\documentclass{article}

\PassOptionsToPackage{numbers, compress}{natbib}

\usepackage[final]{style/neurips_2022}

\usepackage{amsmath,amssymb}

\usepackage[utf8]{inputenc} 
\usepackage[T1]{fontenc}    
\usepackage{hyperref}       
\usepackage{url}            
\usepackage{booktabs}       
\usepackage{amsfonts}       
\usepackage{nicefrac}       
\usepackage{microtype}      
\usepackage{xcolor}         
\usepackage{amsmath}

\usepackage{microtype}
\usepackage{graphicx}
\usepackage{subfigure}
\usepackage{booktabs} 
\usepackage{url}
\usepackage{style/algorithm}
\usepackage{style/algorithmic}
\usepackage{amsmath,amsfonts,bm}
\usepackage{amsthm,amssymb}
\usepackage{bbm}   
\usepackage{diagbox}
\usepackage{booktabs}  
\usepackage{color}
\usepackage{xcolor}

\usepackage{float}
\usepackage{wrapfig}
\usepackage{svg}
\usepackage{lipsum}  
\usepackage{enumitem}

\usepackage{amsmath}
\usepackage{amssymb}
\usepackage{mathtools}
\usepackage{amsthm}

\usepackage[capitalize,noabbrev]{cleveref}  

\theoremstyle{plain}
\newtheorem{theorem}{Theorem}[]
\newtheorem{proposition}[theorem]{Proposition}
\newtheorem*{proposition*}{Proposition}

\theoremstyle{definition}

\newtheorem*{definition*}{Definition}

\theoremstyle{remark}

\newcommand{\KL}{D_{\mathrm{KL}}}

\newcommand{\abs}[1]{\left|#1\right|}

\newcommand{\XXX}{\mathcal{X}} 
\newcommand{\TTT}{\mathcal{T}} 
\newcommand{\SSS}{\mathcal{S}} 
\newcommand{\AAA}{\mathcal{A}} 

\newcommand{\RR}{\mathbb{R}} 
\newcommand{\EE}{\mathbb{E}}

\newcommand{\II}{\mathbb{I}}
\newcommand{\LLL}{\mathcal{L}}
\newcommand{\LFM}{\LLL_{\rm FM}}
\newcommand{\LDB}{\LLL_{\rm DB}}
\newcommand{\LTB}{\LLL_{\rm TB}}
\newcommand{\ra}{{\rightarrow}}

\makeatletter
\renewcommand{\paragraph}
{
  \@startsection{paragraph}{4}{\z@}%
                {0.0ex}%
                {-0.5em}%
        {\normalsize\bf\hspace{-0em}}
}
\makeatother

\title{Trajectory balance:\\Improved credit assignment in GFlowNets}
\author{%
  Esmeralda S. Whitammer\\
  Mila, Universit\'e de Montr\'eal\\
  Montr\'eal, Qu\'ebec, Canada\\
  \And
   Moksh Jain \\
  Mila, Universit\'e de Montr\'eal\\
  Montr\'eal, Qu\'ebec, Canada
  \And
  Emmanuel Bengio \\
  Mila, McGill University, Recursion\\
  Montr\'eal, Qu\'ebec, Canada
  \And
     Chen Sun \\
  Mila, Universit\'e de Montr\'eal\\
  Montr\'eal, Qu\'ebec, Canada
  \And
     Yoshua Bengio \\
  Mila, Universit\'e de Montr\'eal\\
  Montr\'eal, Qu\'ebec, Canada\\
  \small\texttt{\{esmeralda.whitammer,moksh.jain,chen.sun,yoshua.bengio\}@mila.quebec}\\\small\texttt{emmanuel.bengio@recursionpharma.com}
}

\begin{document}

\maketitle



\begin{abstract}
Generative flow networks (GFlowNets) are a method for learning a stochastic policy for generating compositional objects, such as graphs or strings, from a given unnormalized density by sequences of actions, where many possible action sequences may lead to the same object. We find previously proposed learning objectives for GFlowNets, \emph{flow matching} and \emph{detailed balance}, which are analogous to temporal difference learning, to be prone to inefficient credit propagation across long action sequences. 
We thus propose a new learning objective for GFlowNets, \emph{trajectory balance}, as a more efficient alternative to previously used objectives. We prove that any global minimizer of the trajectory balance objective can define a policy that samples exactly from the target distribution. In experiments on four distinct domains, we empirically demonstrate the benefits of the trajectory balance objective  for GFlowNet convergence, diversity of generated samples, and robustness to long action sequences and large action spaces.  
\end{abstract}

\section{Introduction}

Generative flow networks \citep[GFlowNets;][]{bengio2021flow,bengio2021foundations} are models that exploit generalizable structure in an energy function ${\cal E}$ to amortize sampling from the corresponding probability density function on a space of compositional objects $\XXX$, for example, graphs composed of nodes and edges. A GFlowNet learns a stochastic policy that generates such structured objects by producing a stochastic sequence of \emph{actions} that incrementally modify a partial object (\emph{state}), e.g., by adding an edge or a node to a graph, starting from a universal initial state (like an empty graph). A special `exit' action signals that the construction of the object $x\in\XXX$ is completed, and the policy is trained so as to make the likelihood of generating $x$ proportional to the given unnormalized probability or reward $R(x)=e^{-{\cal E}(x)}$.

Like other models in deep reinforcement learning \citep[RL;][]{sutton2018reinforcement}, GFlowNets are trained with a parametric policy that can be given desired inductive biases (e.g., a particular deep net architecture) and allows generalization to states not seen in training. Natural domains for applying GFlowNets are those where exact sampling is intractable and local exploration (MCMC) methods perform poorly, but diversity of samples is desired \cite{bengio2021flow,zhang2022generative,jain2022biological,deleu2022bayesian}. For example, GFlowNets have been used \citep{bengio2021flow} to generate graphical descriptions of molecules by incremental addition of simple building blocks, where the reward $R(x)$ is the estimated strength of binding the constructed molecule to a protein target: the number of candidates grows rapidly with the molecule size and the reward has many separated modes. Like all RL models that iteratively sample action sequences for training, GFlowNets pose the learning challenges of exploration/exploitation and credit assignment, i.e., propagation of a reward signal over an action sequence \citep{van2018deep,bengio2020interference,kumar2021implicit}. \emph{The efficiency of credit assignment and training in GFlowNets is the focus of the present paper.}

The learning problem solved by GFlowNets also has two fundamental differences with the standard reward-maximization paradigm of RL. First, a GFlowNet aims to make the likelihood of reaching a terminating state proportional to the reward, not to concentrate it at a maximal-reward state. Thus, a GFlowNet must model the diversity in the target distribution, not only its dominant mode. Reward maximization in RL can be turned into sampling proportionally to the reward with appropriate entropy maximization regularization, if there is only one way to reach every state~\citep{bengio2021flow}. The second difference with reward-maximization in RL is indeed that the GFlowNet training objectives still lead to correct sampling even when multiple action sequences lead to the same terminating state. Note that the likelihood of reaching a state is the sum of likelihoods of all action sequences leading to it, and that the number of such paths may be exponential in their length.

The set of all achievable sequences of actions and states can be conceptually organized in a  directed graph $G=(\SSS,\AAA)$ in which the vertices $\SSS$ are states (some of them designated as terminal states, in bijection with $\XXX$) and the edges $u{\ra}v$ in $\AAA$ each correspond to applying an action while in a state $u\in \SSS$ and landing in state $v$. In \cite{bengio2021flow}, a GFlowNet is described by a nonnegative function on the edges, called the {\em edge flow} $F:\AAA\to\RR_{\geq0}$, where $F(u\ra v)$ is an \emph{unnormalized} likelihood of taking the action that modifies state $u$ to state $v$. The {\em GFlowNet policy} samples the transition $u\ra v$ from state $u$ with probability $F(u\ra v)/\sum_{v'} F(u\ra v')$ where the denominator sums over the outgoing edges from $u$.  By analogy with the classical notion of flows in networks \citep{ford-fulkerson}, one can think of this flow like the amount of water flowing through an edge (like a pipe) or a state (like a tee, where pipes meet).  It is shown that this GFlowNet policy samples $x$ proportionally to $R(x)$ if $F$ satisfies a set of linear \emph{flow matching} constraints (a conservation law: the sum of flows into a state should equal the sum of flows out of it). These constraints are converted into a temporal difference-like objective that can be optimized with respect to the parameters of a neural net that approximates $F$. An alternative objective based on \emph{detailed balance} constraints was proposed in \cite{bengio2021foundations}. These objectives, however, like temporal-difference learning, can suffer from slow credit assignment \cite{van2018deep,bengio2020interference, kumar2021implicit}.%

The \textbf{main contribution} of this work (\S\ref{sec:traj_balance}) is a new parametrization and objective for GFlowNets. This objective, which we call \emph{trajectory balance}, is computed on sampled full action sequences (trajectories) from the initial state to a terminal state, unlike the flow matching and detailed balance objectives. We prove that global minimization of trajectory balance implies that the learned action policy samples proportionally to $R$.
We also empirically show that {\em the trajectory balance objective accelerates training convergence} relative to previously proposed objectives, improves the learned sampling policy with respect to metrics of diversity and divergence from the reward function, and allows learning GFlowNets that generate sequences far longer than was possible before. 
As a \textbf{secondary contribution}, we perform the first empirical validation of the detailed balance training objective.
Comparative evaluation of the three GFlowNet objectives and non-GFlowNet baselines is performed on four domains illustrating different features of the reward landscape:
\begin{itemize}[leftmargin=*,topsep=0pt,parsep=0pt,itemsep=1mm]
    \item \textbf{Hypergrid} (\S\ref{sec:experiments_hypergrid}), an illustrative synthetic environment with modes separated by wide troughs;
    \item \textbf{Molecule synthesis} (\S\ref{sec:experiments_molecule}), a practical graph generation problem, where the trajectory balance objective leads to significant computational speed-ups and more diverse generated candidates;
    \item \textbf{Sequence generation} (\S\ref{sec:experiments_sequence}), where we show the robustness of trajectory balance to large action spaces and long action sequences on synthetic data and real AMP sequence data.
\end{itemize}
Since the initial appearance of this work on arXiv, several published papers and preprints have used trajectory balance and its generalizations successfully in various applications \citep{zhang2022generative,jain2022biological,deleu2022bayesian,subtb}.

\section{Preliminaries}
\label{sec:preliminaries}

\subsection{Markovian flows}

We give some essential definitions, following \S2 of \cite{bengio2021foundations}.
Fix a directed acyclic graph $G=(\SSS,\AAA)$ with state space $\SSS$ and action space $\AAA$. 
Let $s_0\in \SSS$ be the special \emph{initial (source) state}, the only state with no incoming edges, and designate vertices with no outgoing edges as \emph{terminal} (sinks) \footnote{\cite{bengio2021flow} allowed terminal states with outgoing edges. The difference is easily overcome by augmenting every such state $x$ by a new terminal state $x^\top$ with a stop action $x\ra x^\top$.}. We call the vertices \emph{states}, the edges \emph{actions}, the states reachable through outgoing edges from a state its \emph{children}, and the sources of its incoming edges its \emph{parents}.

A \emph{complete trajectory} is a sequence of transitions \mbox{$\tau=(s_0\ra s_1\ra\dots\ra s_n)$} going from the initial state $s_0$ to a terminal state $s_n$ with $(s_t\ra s_{t+1})\in \AAA$ for all $t$. Let $\TTT$ be the set of complete trajectories. A \emph{trajectory flow} is a nonnegative function $F:\TTT\ra\RR_{\geq0}$. With our water analogy, it could be the number of water molecules travelling along this path (the units don't matter because the flow function can be scaled arbitrarily, since we normalize them to get probabilities).
For any state $s$, define the state flow $F(s)=\sum_{s\in\tau}F(\tau)$, and, for any edge $s\ra s'$, the edge flow 
\begin{equation}
F(s\ra s')=\sum_{\tau=(\dots\ra s\ra s'\ra\dots)}F(\tau).
\end{equation} 
As a consequence of this definition, the \emph{flow matching} constraint (incoming flow = outgoing flow) is satisfied for all states $s$ that are not initial or terminal:
\begin{equation}
F(s)=\sum_{(s''\ra s)\in \AAA}F(s''\ra s)=\sum_{(s\ra s')\in \AAA}F(s\ra s').
\label{eqn:fm_constraint}
\end{equation}
A nontrivial (i.e., not identically zero) trajectory flow $F$ determines a distribution $P$ over trajectories, 
\begin{equation}
    P(\tau)=\frac{1}{Z}F(\tau),\quad Z=F(s_0)=\sum_{\tau\in\TTT}F(\tau).
    \label{eqn:flow_distribution}
\end{equation}
The trajectory flow $F$ is \emph{Markovian} if there exist action distributions $P_F(-|s)$ over the children of each nonterminal state $s$ such that the distribution $P$ has a factorization
\begin{equation}
    P(\tau=(s_0\ra \dots\ra s_n))=\prod_{t=1}^nP_F(s_t|s_{t-1}).
    \label{eqn:flow_pf}
\end{equation}
Equivalently (\cite{bengio2021foundations}, Prop.\ 3) there are distributions $P_B(-|s)$ over the parents of each noninitial state $s$ such that for any terminal $x$,
\begin{equation}
    P(\tau=(s_0\ra \dots\ra s_n)|s_n=x)=\prod_{t=1}^n P_B(s_{t-1}|s_t).
    \label{eqn:flow_pb}
\end{equation}
If $F$ is a Markovian flow, then $P_F$ and $P_B$ can be computed in terms of state and edge flows:
\begin{equation}
    P_F(s'|s)=\frac{F(s\ra s')}{F(s)},\quad P_B(s|s')=\frac{F(s\ra s')}{F(s')},
    \label{eqn:flow_to_policy}
\end{equation}
supposing denominators do not vanish. We call $P_F$ and $P_B$ the \emph{forward policy} and \emph{backward policy} corresponding to $F$, respectively. 
These relations are summarized by the \emph{detailed balance} constraint
\begin{equation}
    F(s)P_F(s'|s)=F(s')P_B(s|s').
    \label{eqn:db_constraint}
\end{equation}
\paragraph{Uniqueness properties.} A Markovian flow is uniquely determined by an edge flow, i.e., a nontrivial choice of nonnegative value on every edge satisfying the flow matching constraint (\ref{eqn:fm_constraint}). By Corollary 1 of \cite{bengio2021foundations}, a Markovian flow is also uniquely determined by either of
\begin{itemize}[leftmargin=*,topsep=-1pt,parsep=0pt,itemsep=1mm]
    \item a constant $Z=F(s_0)>0$ and a distribution $P_F(-|s)$ over children of every nonterminal state; or
    \item a nontrivial choice of nonnegative state flows $F(x)$ for every terminal state $x$ and a choice of distribution $P_B(-|s)$ over parents of every noninitial state.
\end{itemize}

\subsection{GFlowNets}

Suppose that a nontrivial nonnegative reward function \mbox{$R:\XXX\to\RR_{\geq0}$} is given on the set of terminal states. GFlowNets \citep{bengio2021flow} aim to approximate a Markovian flow $F$ on $G$ such that 
\begin{equation}
F(x)=R(x)\quad\forall x\in\XXX.
\label{eqn:reward_matching}
\end{equation}
We adopt the broad definition that a GFlowNet is any learning algorithm consisting of:
\begin{itemize}[leftmargin=*,topsep=-1pt,parsep=0pt,itemsep=1mm]
\item a model capable of providing the initial state flow $Z=F(s_0)$ as well as the forward action distributions $P_F(-|s)$ for any nonterminal state $s$ (and therefore, by the above, uniquely but possibly in an implicit way determining a Markovian flow $F$);
\item an objective function, such that if the model is capable of expressing any action distribution and the objective function is globally minimized, then the constraint (\ref{eqn:reward_matching}) is satisfied for the corresponding Markovian flow $F$.
\end{itemize}
The forward policy of a GFlowNet can be used to sample trajectories from the corresponding Markovian flow $F$ by iteratively taking actions according to policy $P_F(-|s)$. 
If the objective function is globally minimized, then the likelihood of terminating at $x$ is proportional to $R(x)$.

In general, an objective optimizing for (\ref{eqn:reward_matching}) cannot be minimized directly because $F(x)$ is a sum over all trajectories leading to $x$, and computing it may not be practical. Therefore, two local objectives -- \emph{flow matching} and \emph{detailed balance} -- have previously been proposed.


\paragraph{Flow matching objective \citep{bengio2021flow}.} A model $F_\theta(s,s')$ 
\footnote{In practice, it is convenient and more economical to provide a representation of $s$ to the neural net, which simultaneously outputs the flows $F_\theta(s,s')$ for all $s'$ that are reachable by an action from $s$.}
with learnable parameters $\theta$ approximates the edge flows $F(s\ra s')$. The corresponding forward policy is given by $P_F(s'|s;\theta)\propto F_\theta(s,s')$ (Eq. (\ref{eqn:flow_to_policy})). Denote the corresponding Markovian flow by $F_\theta$ and distribution over trajectories by $P_\theta$.
The parameters are trained to minimize the error in the flow matching constraint (\ref{eqn:fm_constraint}) for all noninitial and nonterminal nodes $s$:
\begin{equation}
        \LFM(s)=\left(\log\frac{\sum_{(s''\ra s)\in \AAA}F_\theta(s'',s)}{\sum_{(s\ra s')\in \AAA}F_\theta(s,s')}\right)^2
    \label{eqn:fm_objective_node}
\end{equation}
and a similar objective $\LFM'$ pushing the inflow at $x\in\XXX$ to equal $R(x)$ at terminal nodes $x$. This objective is optimized for nonterminal states $s$ and terminal states $x$ from trajectories sampled from a training policy $\pi_\theta$. Usually, $\pi_\theta$ is chosen to be a tempered (higher temperature) version of $P_F(-|s,\theta)$, which also helps exploration during training. The parameters are updated with stochastic gradient 
\begin{equation}
    \EE_{\tau=(s_0\ra\dots\ra s_n)\sim \pi_\theta}\nabla_\theta\hspace{-1mm}\left[\sum_{t=1}^{n-1}\LFM(s_t)+\LFM'(s_n)\right]\hspace{-1mm}.
\end{equation}
As per Proposition 10 of \cite{bengio2021foundations}, if the training policy $\pi_\theta$ has full support, and a global minimum of the expected loss (\ref{eqn:fm_objective_node}) over states on trajectories sampled from $\pi_\theta$ is reached, then the GFlowNet samples from the target distribution (i.e., $F_\theta$ satisfies (\ref{eqn:reward_matching})).


\paragraph{Detailed balance objective \citep{bengio2021foundations}.} A neural network model with parameters $\theta$ has input $s$ and three kinds of outputs: an estimated state flow $F_\theta(s)$, an estimated distribution over children $P_F(-|s;\theta)$, and an estimated distribution over parents $P_B(-|s;\theta)$. The policy $P_F(-|-;\theta)$ and the initial state flow $F_\theta(s_0)$ uniquely determine a Markovian flow $F_\theta$, which is not necessarily compatible with the estimated backward policy $P_B(-|-;\theta)$. The error in the detailed balance constraint (\ref{eqn:db_constraint}) is optimized on actions $(s\ra s')$ between nonterminal nodes seen along trajectories sampled from the training policy:
\begin{equation}
    \LDB(s,s')=\left(\log \frac{F_\theta(s)P_F(s'|s;\theta)}{F_\theta(s')P_B(s|s';\theta)}\right)^2,
    \label{eqn:db_objective_edge}
\end{equation}
and a similar constraint $\LDB'(s,x)$ is optimized at actions leading to terminal nodes. Similarly to flow matching, the parameters are updated with stochastic gradient
\begin{equation}
    \EE_{(s_0\ra\dots\ra s_n)\sim \pi_\theta}\nabla_\theta\left[\sum_{t=1}^{n-1}\LDB(s_{t-1},s_t)+\LDB'(s_{n-1},s_n)\right]
\end{equation}
along trajectories sampled from a training policy $\pi_\theta$. By Proposition 6 of \cite{bengio2021foundations}, a global minimum of the expected detailed balance loss under a $\pi_\theta$ with full support specifies a GFlowNet that samples from the target distribution, i.e., the flow $F_\theta$ satisfies (\ref{eqn:reward_matching}).


\paragraph{Remarks.}
In some problems, such as autoregressive sequence generation (\S\ref{sec:experiments_sequence}), the directed graph $G$ is a tree, so each state has only one parent. In this case, $P_B$ is trivial and the detailed balance objective reduces to the flow matching objective, which in turn can be shown to be equivalent to Soft Q-Learning \citep{haarnoja2017reinforcement,buesing2019approximate} with temperature $\alpha=1$, a uniform $q_{\mathbf{a}'}$, and $\gamma=1$.

\section{Trajectory balance}
\label{sec:traj_balance}

Let $F$ be a Markovian flow and $P$ the corresponding distribution over complete trajectories, defined by (\ref{eqn:flow_distribution}), and let $P_F$ and $P_B$ be forward and backward policies determined by $F$. A direct algebraic manipulation of Eqs.~(\ref{eqn:flow_distribution},\ref{eqn:flow_pf},\ref{eqn:flow_pb}) gives the \emph{trajectory balance constraint} for any complete trajectory $\tau=(s_0\ra s_1\ra\dots\ra s_n=x)$:
\begin{equation}
    Z\prod_{t=1}^nP_F(s_t|s_{t-1})=F(x)\prod_{t=1}^nP_B(s_{t-1}|s_t),
    \label{eqn:tb_constraint}
\end{equation}
where we have used that $P(s_n=x)=\frac{F(x)}{Z}$.

As explained in \S\ref{app:extensions}, the trajectory balance constraint (\ref{eqn:tb_constraint}) and the detailed balance constraint (\ref{eqn:db_constraint}) are special cases of one general constraint, which has been studied as a training objective in \cite{subtb}.

\paragraph{Trajectory balance as an objective.} We propose to convert (\ref{eqn:tb_constraint}) into an objective to be optimized along trajectories sampled from a training policy. Suppose that a model with parameters $\theta$ outputs estimated forward policy $P_F(-|s;\theta)$ and backward policy $P_B(-|s;\theta)$ for states $s$ (just as for detailed balance above), as well as a global scalar $Z_\theta$ estimating $F(s_0)$. The scalar $Z_\theta$ and forward policy $P_F(-|-;\theta)$ uniquely determine an implicit Markovian flow $F_\theta$.


\begin{algorithm}[t]
\begin{algorithmic}[1]
\INPUT Reward function $R:\XXX\to\RR_{>0}$, model and optimizer hyperparameters
\STATE Initialize models $P_F,P_B,Z$ with parameters $\theta$
\REPEAT
\STATE Sample trajectory $\tau=(s_0\ra\dots\to s_n)$ from policy $P_F(-|-;\theta)$ or a tempered version of it
\STATE $\theta\leftarrow\theta-\eta\nabla_\theta\LTB(\tau)$ \COMMENT{gradient update on (\ref{eqn:tb_objective_one})}
\UNTIL 
convergence monitoring on running $\LTB(\tau)$
\end{algorithmic}
\caption{Training a GFlowNet with trajectory balance}
\label{alg:traj_balance}
\end{algorithm}

For a trajectory $\tau=(s_0\ra s_1\ra\dots\ra s_n=x)$, define the \emph{trajectory loss}
\begin{equation}
    \LTB(\tau)=\left(\log\frac{Z_\theta\prod_{t=1}^nP_F(s_t|s_{t-1};\theta)}{R(x)\prod_{t=1}^nP_B(s_{t-1}|s_t;\theta)}\right)^2.
    \label{eqn:tb_objective_one}
\end{equation}
If $\pi_\theta$ is a training policy -- usually that given by $P_F(-|-;\theta)$ or a tempered version of it -- then the trajectory loss is updated along trajectories sampled from $\pi_\theta$, i.e., with stochastic gradient
\begin{equation}
    \EE_{\tau\sim\pi_\theta}\nabla_\theta\LTB(\tau).
    \label{eqn:tb_objective}
\end{equation}
The full algorithm, with batch size of 1, is presented as Algorithm~\ref{alg:traj_balance} and its 
correctness is guaranteed by the following.

\newcommand{\mainproposition}{
Let $R$ be a positive reward function on $\XXX$.
\begin{enumerate}[label=(\alph*), itemsep=0pt,parsep=0pt,topsep=0pt,leftmargin=*]
\item If $P_F(-|-;\theta)$, $P_B(-|-;\theta)$, and $Z_\theta$ are the forward and backward policies and normalizing constant of a Markovian flow $F$ satisfying (\ref{eqn:reward_matching}), then $\LTB(\tau)=0$ for all complete trajectories $\tau$.
\item Conversely, suppose that $\LTB(\tau)=0$ for all complete trajectories $\tau$. Then the corresponding Markovian flow $F_\theta$ satisfies (\ref{eqn:reward_matching}), and $P_F(-|-;\theta)$ samples proportionally to the reward.
\end{enumerate}
}
\begin{proposition}
\label{prop:main}
\mainproposition
\end{proposition}
The proof is given in \S\ref{app:proofs}. In particular, if $\pi_\theta$ has full support and $\EE_{\tau\sim\pi_\theta}\LTB(\tau)$ is globally minimized over all forward and backward policies $(P_F,P_B)$ and normalizing constants $Z$, then the corresponding Markovian flow $F_\theta$ satisfies (\ref{eqn:reward_matching}) and $P_F(-|-;\theta)$ samples proportionally to the reward. (The positivity assumption on $R$ is necessary to avoid division by 0 in (\ref{eqn:tb_objective_one}), but can be relaxed by introduction of smoothing constants, just as was done for the losses proposed in \cite{bengio2021flow,bengio2021foundations}.)

\paragraph{Remarks.} \textbf{(1)} 
As discussed in \S\ref{sec:preliminaries}, in the case of auto-regressive generation, $G$ is a directed tree, where each $s\in \mathcal{S}$ has a single parent state. In this case $P_B$ is trivially $P_B=1$, $\forall s \in \mathcal{S}$. We get 
\begin{equation}
    \LTB(\tau) = \left(\log \frac{Z_\theta\prod_{t=1}^n P_F(s_t|s_{t-1}; \theta)}{R(x)}\right)^2
\end{equation}
\textbf{(2)} We found it beneficial to parametrize $Z$ in the logarithmic domain ($\log Z$ is the trainable parameter) and output logits for $P_F(-|s;\theta)$ and $P_B(-|s;\theta)$, so that all products in (\ref{eqn:tb_objective_one}) become sums under the logarithm. This is consistent with the log-domain parametrization of flows in \cite{bengio2021flow}.
In addition, we found it helpful to set a higher learning rate for $Z$ than for the parameters of $P_F$ and $P_B$.\footnote{Because the loss (\ref{eqn:tb_objective_one}) is quadratic in $\log Z$, gradient updates on $\log Z$ are equivalent to setting it to a weighted moving average of the discrepancy between $\log\prod P_F$ and $\log(R(x)\prod P_B)$. Optimizers enhanced with momentum complicate things: we leave empirical investigation of these questions to future work.}

\subsection{Canonical choice of reward-matching flow}
\label{sec:canonical_solution}

The constraint (\ref{eqn:reward_matching}), in general, does not have a unique solution: if the underlying undirected graph of $G$ has cycles, there may be multiple Markovian flows whose corresponding action policies sample proportionally to the reward. However, by the uniqueness properties, for any choice of backward policy $P_B(-|-)$, there is a unique flow satisfying (\ref{eqn:reward_matching}), and thus a unique corresponding forward policy $P_F(-|-)$ for states with nonzero flow. (See Fig.~\ref{fig:grid_uniform_learned}.)

In some settings, it may be beneficial to \emph{fix} the backward policy $P_B$ and train only the parameters giving $P_F$ and $Z_\theta$. For example, it may difficult to construct a model that outputs a distribution over the parents of a given input state (e.g., for the molecule domain (\S\ref{sec:experiments_molecule}), it is hard to force invariance to molecule isomorphism). A natural choice is to set $P_B(-|s)$ to be uniform over all the parents of a state $s$, i.e., $P_B(-|s)=1/\#\{s'\mid(s'\ra s)\in\AAA\}$. 

\section{Related work}

\paragraph{Reinforcement learning.} 
GFlowNets are trained to sample proportionally the reward rather than maximize it as usual in RL. However, on tree-structured DAGs (autoregressive generation) are equivalent to RL with appropriate entropy regularization or soft Q-learning and control as inference \citep{buesing2019approximate,haarnoja2017reinforcement, haarnoja2018soft}. The experiments and discussion of~\cite{bengio2021flow} show how these methods can fail badly in the general DAG case well handled by GFlowNets. Signal propagation over sequences of several actions in trajectory balance is also related to losses used in RL computed on subtrajectories \citep{nachum2017bridging}.

\paragraph{Local exploration vs.\ amortized generalization to unseen modes.}
GFlowNets are also related to MCMC methods for sampling from unnormalized densities. While there has been work on accelerating or partially amortizing sampling from unnormalized densities over discrete spaces when exact sampling is intractable \citep{grathwohl2021oops,dai2020aloe}, some of it domain- or problem-specific \citep{xie2021mars}, GFlowNets treat the compositional structure in data as a learning problem (enabling generalization to unseen modes), not as a bias to build in to the sampler. Thus, the cost is amortized and borne by the learning of that structure through sampling, not through search at generation time.

\paragraph{Variational inference.} 
GFlowNets are connected with variational methods for fitting hierarchical generative models. The squared log-ratio loss proposed in \cite{mnih2014neural} as a control variate in the optimization of an evidence lower bound can be seen as a special case of trajectory balance. See \S\ref{app:vi} for further discussion, in which we prove that on-policy optimization of trajectory balance is equivalent to minimization of a certain KL divergence. Two recent papers  \cite{unifying,gfn-hvi} extend our observations.

\section{Experiments}

We evaluate the proposed trajectory balance objective against prior objectives for training GFlowNets as well as standard methods for learning policies that approximately sample objects proportionally to their rewards, like MCMC as well as against other RL techniques. Our experiments include the hypergrid and molecule synthesis tasks from \cite{bengio2021flow} and two new tasks in which $G$ is a directed tree. 

\subsection{Hypergrid environment}
\label{sec:experiments_hypergrid}

In this subsection, we study a synthetic hypergrid environment introduced by \cite{bengio2021flow}. This task is easier than others we study, but we include it for completeness, and because it allows us to illustrate some interesting behaviours.\footnote{Code: \url{https://gist.github.com/malkin1729/9a87ce4f19acdc2c24225782a8b81c15}.}

In this environment, the nonterminal states $\SSS^\circ$ form a $D$-dimensional hypergrid with side length $H$:
\begin{equation*}
\SSS^\circ=\{(s^1,\dots,s^D)\mid s^d\in\{0,1,\dots,H-1\},d=1,\dots,D\},
\end{equation*}
and actions are operations of incrementing one coordinate in a state by 1 without exiting the grid.
The initial state is $(0,\dots,0)$. For every nonterminal state $s$, there is also a termination action that transitions to a corresponding terminal state $s^\top$ (cf.\ footnote 1). The reward at a terminal state $s^\top=(s^1,\dots,s^d)^\top$ is given by
\begin{align*}
    R(s^\top)=R_0 &+ 0.5\prod_{d=1}^D\II\left[\abs{\frac{s^d}{H-1}-0.5}\in(0.25,0.5]\right] \nonumber + 2\prod_{d=1}^D\II\left[\abs{\frac{s^d}{H-1}-0.5}\in(0.3,0.4)\right]\hspace{-2mm}
\end{align*}
where $\II$ is an indicator function and $R_0$ is a constant controlling the difficulty of exploration. This reward has peaks of height $2.5+R_0$ near the corners of the hypergrid, surrounded by plateaux of height $0.5+R_0$. These plateaux are separated by wide troughs with reward $R_0$. An illustration with $H=8$ and $D=2$ is shown in the left panel of Fig.~\ref{fig:grid_uniform_learned}. This environment evaluates the ability of a GFlowNet to generalize from visited states to infer the existence of yet-unvisited modes.

We train GFlowNets with the detailed balance (DB) and trajectory balance (TB) objectives with different $H$, $D$, and $R_0$, in addition to reproducing the flow matching (FM) experiments and non-GFlowNet baselines based upon \cite{bengio2021flow}'s published code. Our GFlowNet policy model is a multilayer perceptron (MLP) that accepts as input a one-hot encoding of a state $s$ (with the goal of enabling generalization) and outputs logits of the forward and backward policies $P_F(-|-;\theta)$ and $P_B(-|-;\theta)$ (as well as the estimated state flow $F_\theta(s)$ in the case of DB). The forward policy, backward policy, and state flow models share all but the last weight matrix of the MLP. This is consistent with 
\cite{bengio2021flow}'s model, where an identical architecture accepted $s$ as input and output estimated flows $F_\theta(s,s')$ for all children $s'$ of $s$. Details are given in \S\ref{app:grid}.

We consider a 4-dimensional grid with $H=8$ and and a 2-dimensional grid with $H=64$. The two grids have the same number of terminal states, but the 2-dimensional grid has longer expected trajectory lengths. For both grid sizes, we consider $R_0=0.1,0.01,0.001$, with smaller $R_0$ giving environments that are more difficult to explore due to the lower likelihood for models to cross the low-reward valley. For the models trained with DB and TB, we also explore the effect of fixing the backward policy to be uniform (\S\ref{sec:canonical_solution}).

\paragraph{Results.}
In Fig.~\ref{fig:grid_l1_curves}, we plot the evolution over the course of training of the $L_1$ error between the true reward distribution (the reward $R(x)$ normalized over all possible terminal states $x\in\XXX$) and the empirical distribution of the last $2\cdot10^5$ visited states for all settings (which would have a probability of 0 on $x$'s not visited).
Although convergence to the same stable minimum is achieved by all models and settings, DB and TB training tend to converge faster than FM, with a slight benefit of TB over detailed balance in the 4-D environment. 

\paragraph{Effect of uniform $P_B$.}
Note the difference in learning speed between models with fixed uniform backward policy $P_B$ and models with learned $P_B$. As noted in \S\ref{sec:canonical_solution}, when $P_B$ is fixed, there is a unique $P_F(-|-;\theta)$ that globally minimizes the objective, and it  may be approached slowly. However, if $P_B$ and $P_F$ are permitted to evolve jointly, they may more quickly approach one of the many optimal solutions. This is confirmed by the much faster convergence of models with learned $P_B$ on the $64\times64$ grid.
We have observed that, especially for large grid sizes, when $P_B$ and $P_F$ are both learned, the model has a bias towards first taking all actions in one coordinate direction, then proceeding in the other direction until terminating (as in the right panel of Fig.~\ref{fig:grid_uniform_learned}), perhaps because a constant distribution over two actions (`continue to the right' and `terminate') can be modeled with higher precision over a large portion of the grid than the complex position-dependent distribution as shown in the centre panel of Fig.~\ref{fig:grid_uniform_learned}.

\begin{figure}[t]
    \centering
    \includegraphics[width=0.256\linewidth]{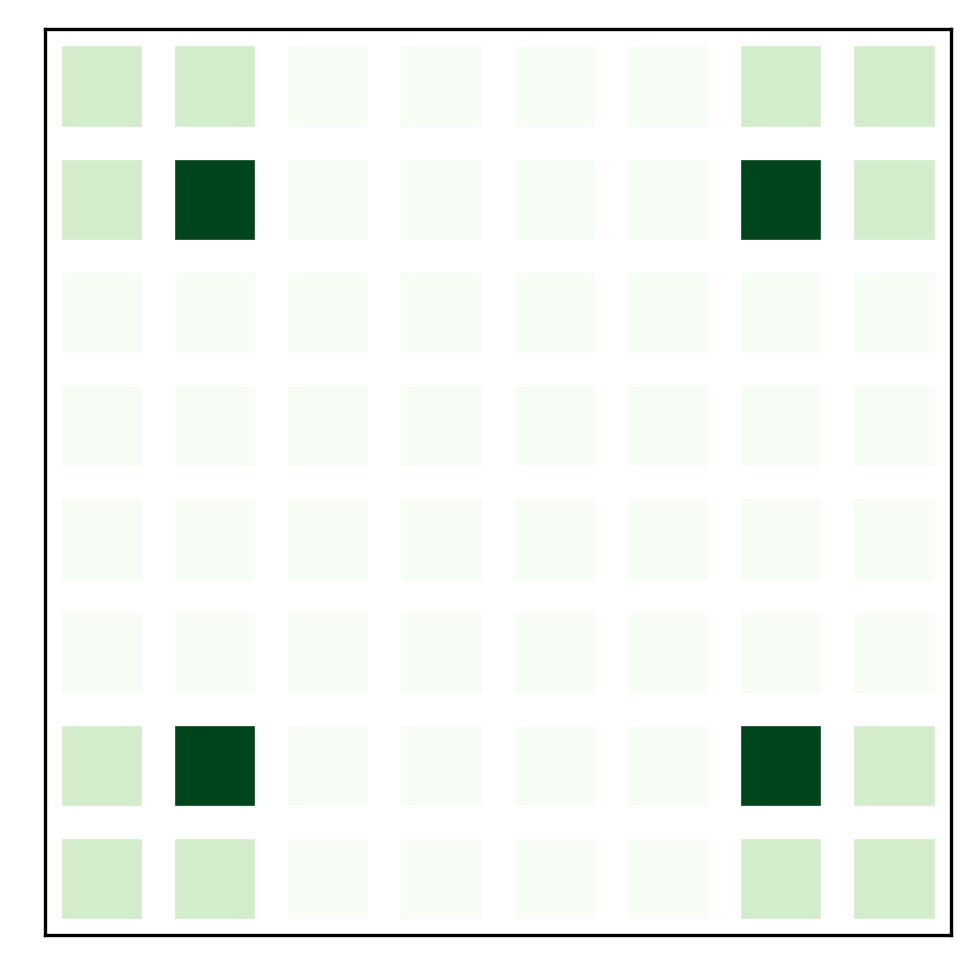}\hfill
    \raisebox{0.1cm}{\includegraphics[width=0.255\linewidth,trim=12 12 10 10,clip]{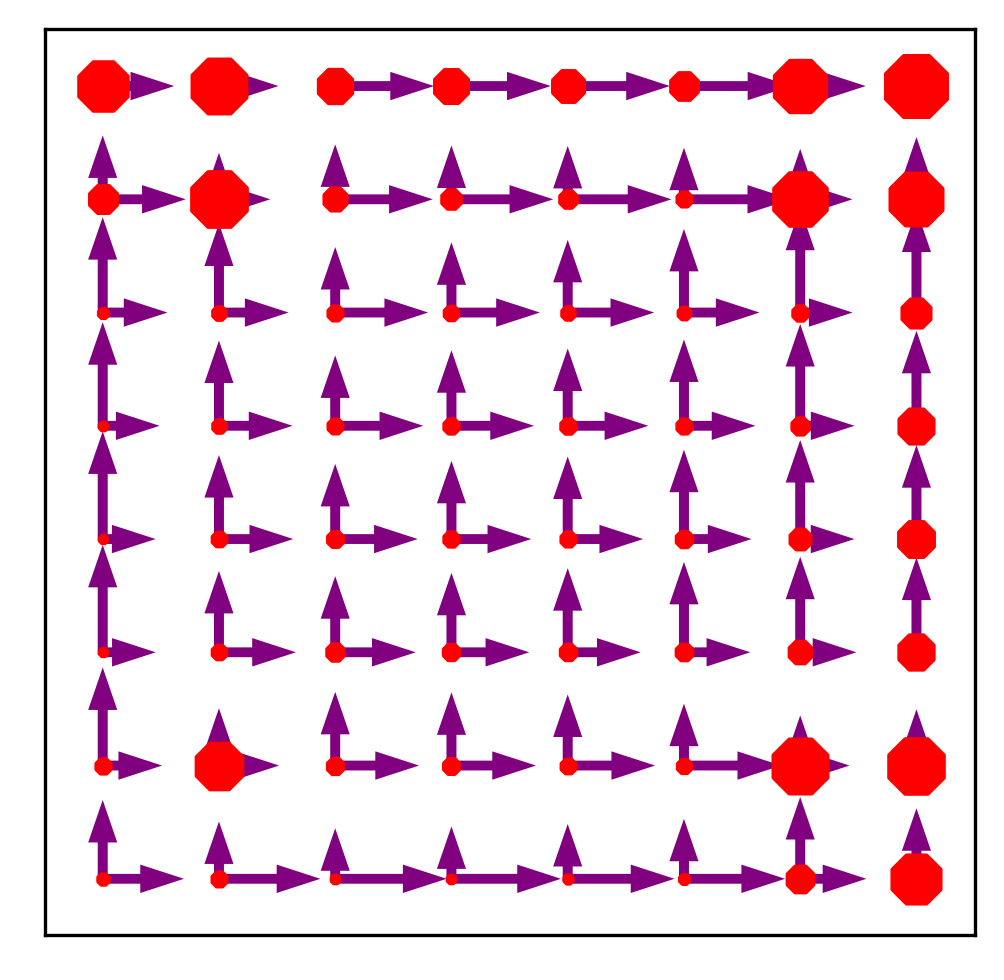}}\hfill
    \raisebox{0.1cm}{\includegraphics[width=0.255\linewidth,trim=12 12 10 10,clip]{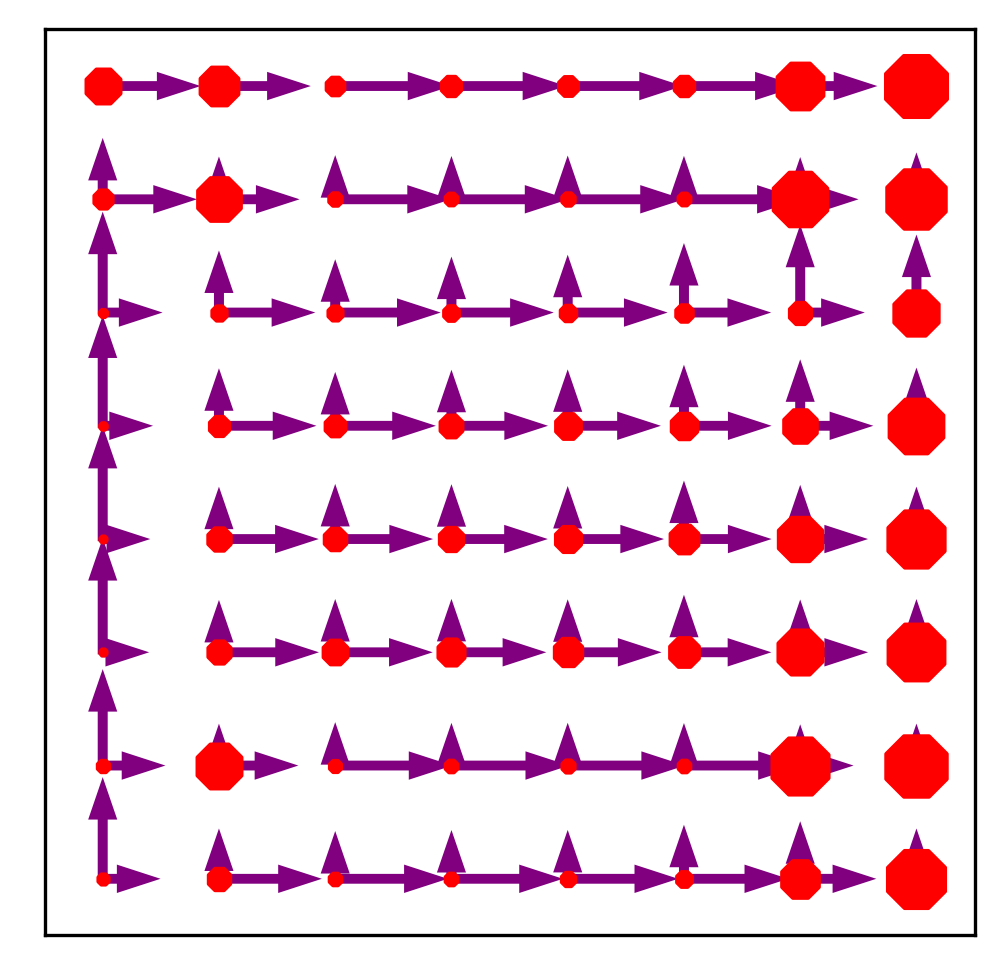}}
    \caption{\textit{Left:} The reward function on an $8\times8$ grid environment (\S\ref{sec:experiments_hypergrid}) with $R_0=0.1$. \textit{Centre and right:} Two forward action policies -- with fixed uniform $P_B$ and with a learned non-uniform $P_B$ -- that sample from this reward. The lengths of arrows pointing up and right from each state are proportional to the likelihoods of the corresponding actions under $P_F$, and the sizes of the red octagons are proportional to the termination action likelihoods.}
    \label{fig:grid_uniform_learned}
\end{figure}

\begin{figure}[t]
\centering
\includegraphics[width=0.49\linewidth,trim=0 10 0 0,clip]{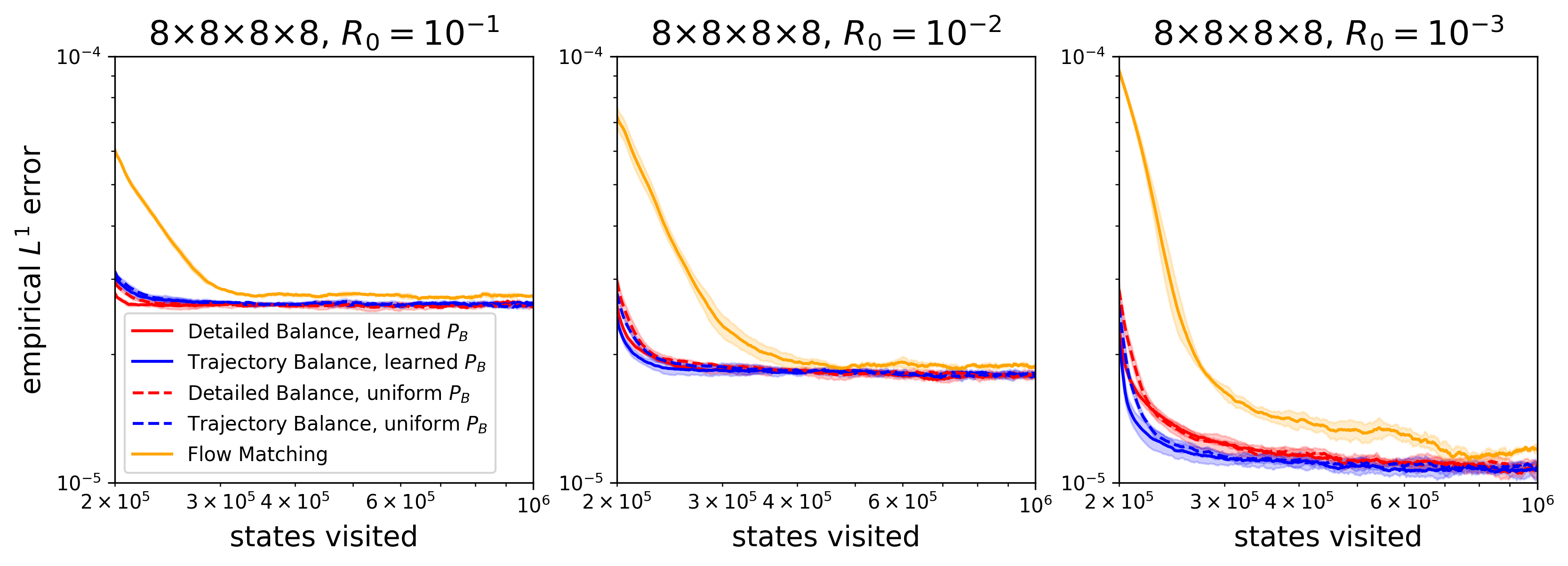}
\includegraphics[width=0.49\linewidth,trim=0 10 0 0,clip]{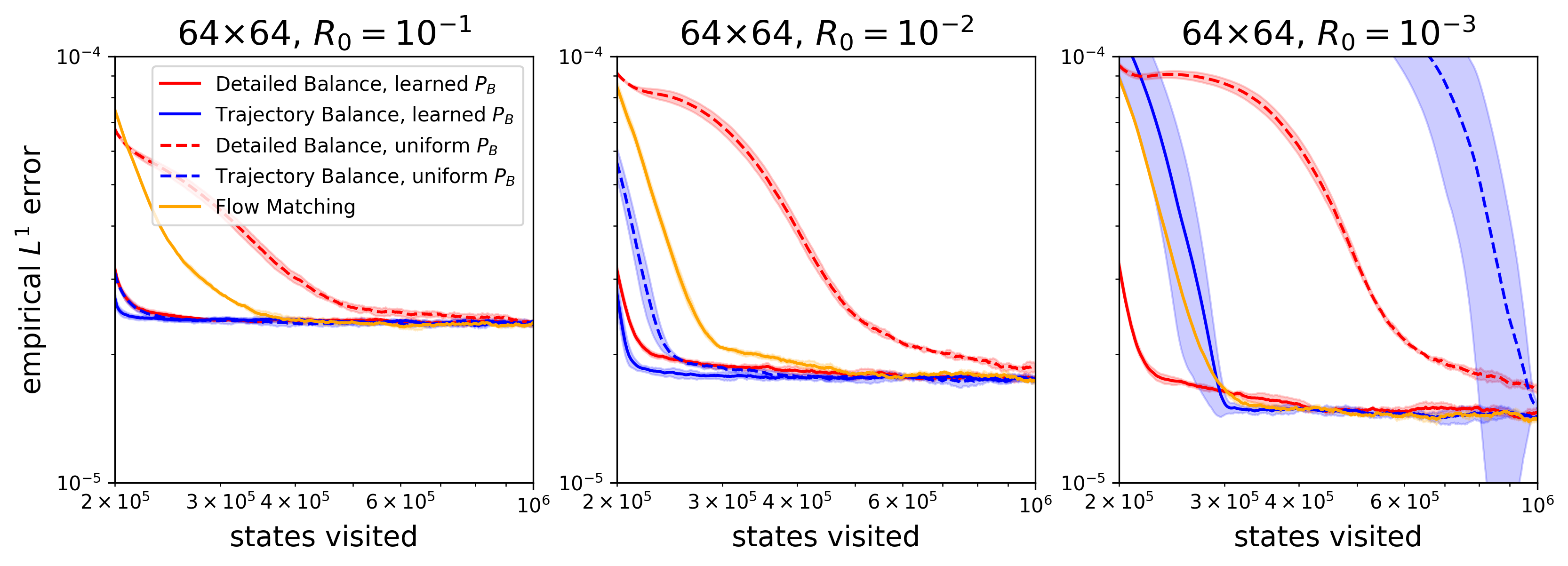}
    \caption{Empirical $L^1$ error between true and sampled state distributions on the grid environment with varying grid size and $R_0$. Mean and standard error over 5 seeds. The curves for PPO and MCMC baseline would lie outside the plot bounds.}
    \label{fig:grid_l1_curves}
\end{figure}

\subsection{Small drug molecule synthesis}
\label{sec:experiments_molecule}
Next, we consider the molecule generation task~\citep{xie2021mars, Jin_2020, kumar2012fragment, gilmer2017neural, shi2021masked} introduced for GFlowNets in \cite{bengio2021flow}. We extend \cite{bengio2021flow}'s published code with an implementation of the TB and DB objectives.\footnote{Code: \url{https://github.com/GFNOrg/gflownet/tree/trajectory_balance}.}

\begin{figure}[t]
    \centering
    \includegraphics[width=\linewidth,trim=50 10 50 0,clip]{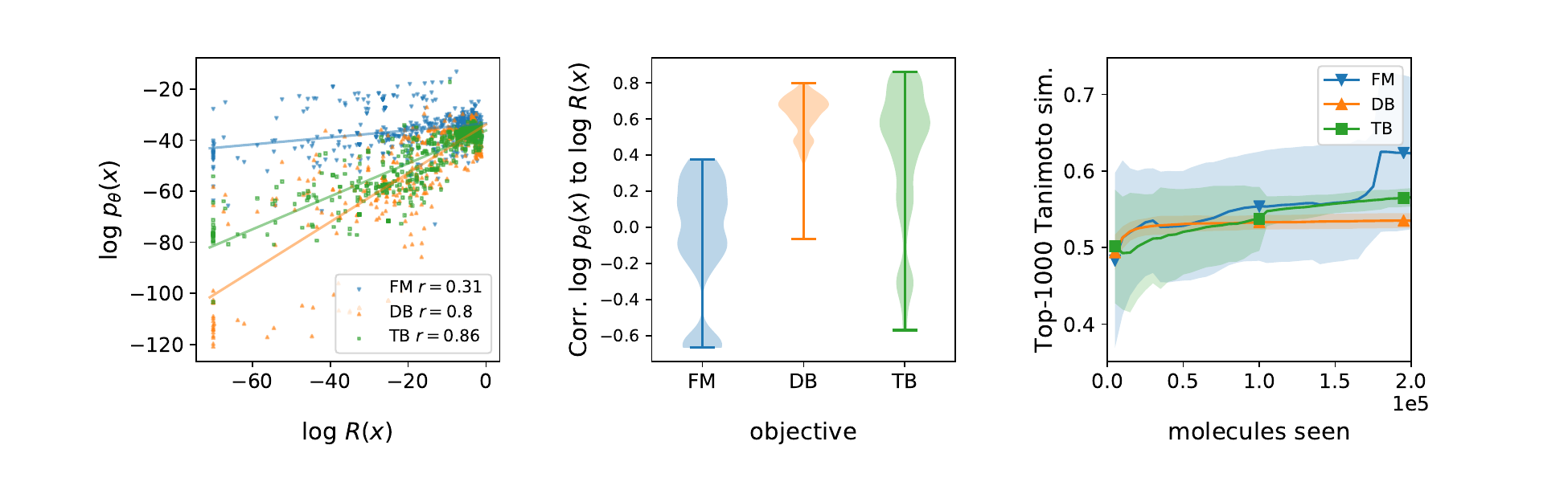}\vskip-3mm
    \caption{\textit{Left, Centre:} Pearson correlations between rewards and sampling probability. $\log p_\theta(x)$ is the log-likelihood that a trajectory sampled from the learned policy $P_F(-|-;\theta)$ terminates at $x$. 
    \textit{Left:} Scatter plot on a test set of $x$'s for the best hyperparameters of TB, FM, and DB. \textit{Centre:} Violin plot of correlations for 16 hyperparameter settings and 3 seeds for each setting, showing TB being capable of fitting better.
    \textit{Right:} Average pairwise Tanimoto similarity for the top 1000 samples generated by GFlowNets as training progresses. Lines are the average across runs, shaded regions the standard deviation. Models trained with TB have consistently lower similarity than those with FM, hence greater diversity. We hypothesize that the higher variance, in correlation and diversity, of TB relative to DB is related to high variance of the stochastic gradient; see \cite{subtb}.}
    \label{fig:corr}
\end{figure}

The goal is to generate molecules, in the form of graphs, with a low binding energy to the 4JNC inhibitor of the sEH (soluble epoxide hydrolase) protein. The graphs generated are junction trees~\citep{Jin_2020} of a vocabulary of building blocks. The reward is defined as the normalized negative binding energy as predicted by a \emph{proxy} model, itself trained to predict energies computed via docking simulations~\citep{trott2010autodock}. The maximum trajectory length is 8, with the number of actions varying between around 100 and 2000 (the larger a molecule, the more possible additions exist), making $|{\cal X}|$ about $10^{16}$. 

\paragraph{Results.}
We plot in Fig.~\ref{fig:corr} (left and centre) the correlation of log-reward and log-sampling probability (the likelihood that a trajectory sampled from the learned policy terminates at $x$) for GFlowNets trained using TB, FM and DB. This correlation is significantly higher for models trained with TB. The points $x$ shown are from a fixed held-out set to which the GFlowNets do not have access in training. Note that a perfect model would have correlation 1, as $\log R(x)$ and $\log p_\theta(x)$ would differ by a constant (equal to $\log Z$) that is independent of $x$.

In Fig.~\ref{fig:corr} (right) we plot the average pairwise Tanimoto similarity~\citep{bender2004molecular} 
for the 1000 highest-reward samples generated over the course of training. We see that TB consistently generates more diverse molecules than FM. These results showcase the benefits of TB, not only for faster temporal credit assignment, but for generalization and diversity. In addition, TB has up to $5\times$ runtime speedup over FM as the enumeration of parents is not needed. See \S\ref{app:mol_synth} for further discussion.



\subsection{Autoregressive sequence generation}
\label{sec:experiments_sequence}

Finally, we evaluate the TB objective on the task of autoregressive sequence generation. In \S\ref{sec:bit_sequences}, we study the effect of trajectory length and action space size on the learning dynamics in GFlowNets.  In \S\ref{sec:amps}, we consider the more realistic task of generating peptides (short protein sequences with anti-microbial properties) and evaluate GFlowNets against standard RL and MCMC baselines. 

\subsubsection{Bit sequences}
\label{sec:bit_sequences}

\paragraph{Task.} 
The goal is to generate bit sequences  of a fixed length $n=120$ ($\mathcal{X} = \{0, 1\}^{n}$), where the reward is designed to have modes at a given fixed set $M \subset \mathcal{X}$ that is unknown to the learner.
The reward for a sequence $x$ is defined as $R(x) = \exp(1-\min_{y \in M}d(x, y)/n)$, where $d$ is the edit distance. 
We describe the procedure to generate $M$  in \S\ref{app:bit_sequence}. 

For different integers $k$ dividing $n$, we design action spaces for left-to-right generation of sequences in $\XXX$, where a complete trajectory has $\frac nk$ actions and each action appends a $k$-bit `word' to the end of a partial sequence. A forward policy on this state space is the same an autoregressive sequence model over a vocabulary of size $2^k$. Varying $k$ while fixing $\mathcal{X}$ and $M$ allows us to study the effect of the tradeoff between trajectory lengths ($\frac nk$) and the action space sizes ($|V|=2^k$) without changing the underlying probabilistic modeling problem.

We compare GFlowNets trained with the TB objective against GFlowNets trained with the FM objective (equivalent to DB and Soft Q-Learning in this case) and two non-GFlowNet baselines: A2C with Entropy Regularization \cite{williams1991function,mnih2016asynchronous}, Soft Actor-Critic~\cite{haarnoja2018soft,christodoulou2019soft} and MARS~\citep{xie2021mars}. We use a Transformer-based architecture \cite{vaswani2017attention} across all the methods. See \S\ref{app:bit_sequence} for details.

To evaluate the methods we use (1) Spearman correlation between the probability of generating the sequence $p(x)=F(x)/Z$ and its reward $R(x)$ on a test set sampled approximately uniformly over the possible values of the reward, (2) number of modes captured (number of reference sequences from $M$ for which a candidate within a distance $\delta$ has been generated). In our experiments, \mbox{$n=120$}, \mbox{$|M|=60$}, \mbox{$k \in\{1,2,4,6,8,10\}$}, and \mbox{$\delta=28$}.

\paragraph{Results.}
Fig.~\ref{fig:bits} (left) presents the results for the Spearman correlation evaluation. We observe that GFlowNets trained with the TB objective learn policies that correlate best with the reward $R(x)$ across all action spaces. In particular, we observe the effect of inefficient credit assignment in GFlowNets trained with FM, as the correlation improves with increasing $k$, i.e., shorter trajectories. On the other hand, large action spaces also hurt GFlowNets trained with the FM objective, while the TB objective is robust to them. Additionally, we can observe in Fig.~\ref{fig:bits} (right) that for fixed $k$, GFlowNets trained with TB discover more modes faster than other methods.

\begin{figure}[t]
    \centering
    \includegraphics[height=0.28\textwidth,trim=0 10 200 0,clip]{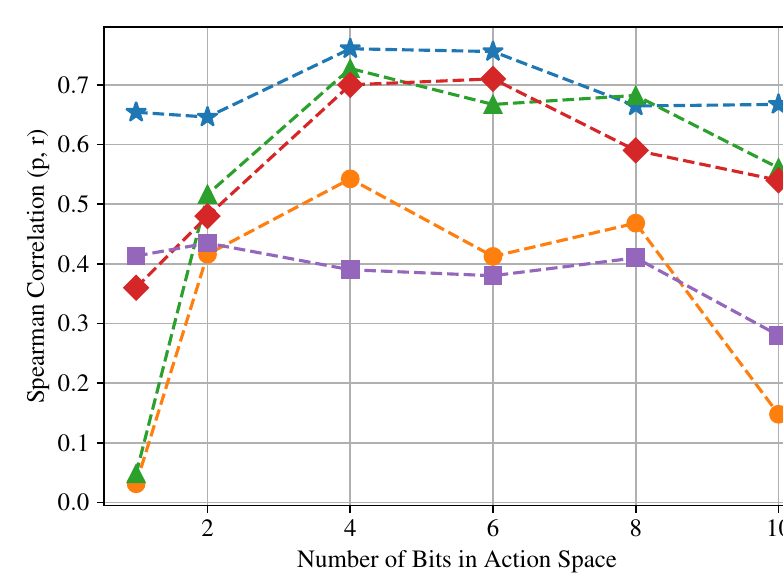}
    \includegraphics[height=0.28\textwidth,trim=0 10 0 0,clip]{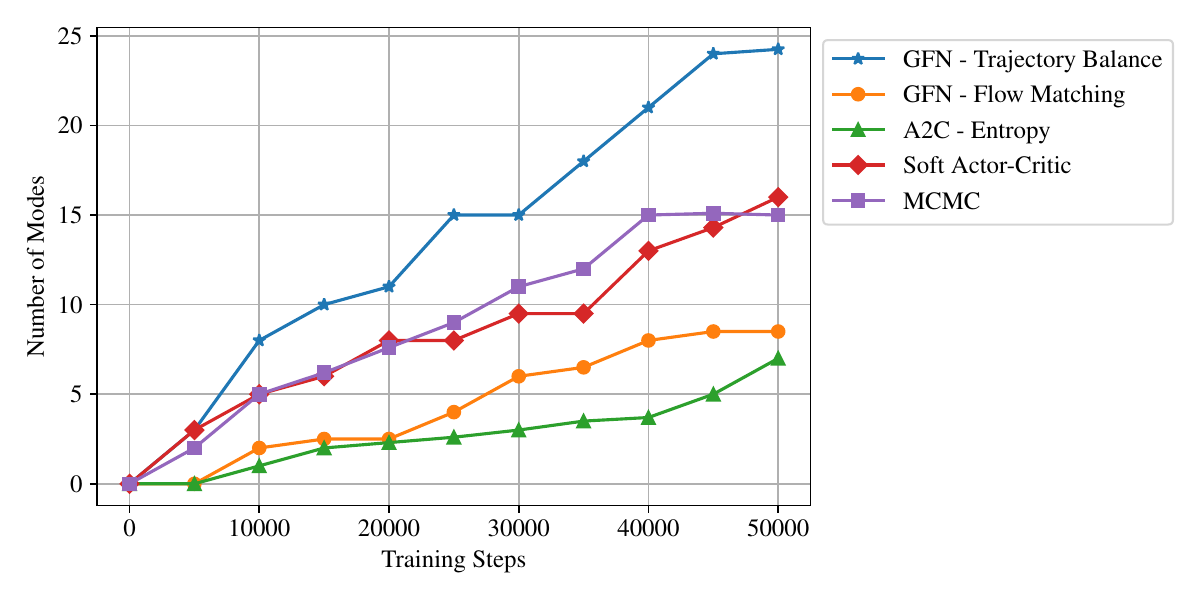}
    \caption{\textit{Left:} Spearman correlation of the sampling probability under different learned policies and reward on a test set, plotted against the number of bits $k$ in the symbols in $V$ in the bit sequence generation task. GFlowNets trained with trajectory balance learn policies that have the highest correlation with the reward $R(x)$ and are robust to length and vocabulary size. \textit{Right:} Number of modes discovered over the course of training on the bit sequence generation task with $k=1$. GFlowNets trained with trajectory balance discover more modes faster.}
    \label{fig:bits}
\end{figure}

\subsubsection{Anti-Microbial Peptide (AMP) generation}
\label{sec:amps}

\begin{wraptable}[4]{r}{0.45\textwidth}
\vspace{-3.75em}
    \caption{Results on the AMP generation task.}
\resizebox{0.45\textwidth}{!}{
\begin{tabular}{@{}lll@{}}
\toprule
                   & Top 100 Reward & Top 100 Diversity \\ \midrule
GFN-$\LTB$  & ${\bf 0.85} \pm 0.03$ & $\bf 18.35$ $\pm 1.65$\\
GFN-$\LFM$/$\LDB$  & $0.78 \pm 0.05$ & $12.61 \pm 1.32$ \\
SAC & $0.80\pm 0.01$ & $8.36 \pm 1.44$\\
AAC-ER    & $0.79 \pm 0.02$ & $7.32 \pm 0.76$ \\
MCMC      & $0.75 \pm 0.02$  & $12.56 \pm 1.45$  \\
\bottomrule
\end{tabular}
}
    \label{tab:amp_results}
\end{wraptable}

In this section, we consider the practical task of generating peptide sequences that have anti-microbial activity. The goal is to generate a protein sequence (where the vocabulary consists of 20 amino acids and a special end-of-sequence action), with maximum length $60$. We take $6438$ known AMP sequences and $9522$ non-AMP sequences from the DBAASP database~\cite{pirtskhalava2021dbaasp}. We then train a classifier on this dataset, using $20\%$ of the data as a validation set. The probability output by this model for a sequence to be classified as an AMP is used as the reward $R(x)$ in our experiments. 

The state and actions are designed just as in the previous experiment, with each action appending a symbol to the right of a state. We again compare TB and FM, as well as A2C with entropy regularization, SAC and MCMC as baselines. We again use Transformers for all the experiments on this task; see further details in Appendix~\ref{app:amp_gen}.
We generate $2048$ sequences from each method, and pick the top 100 sequences ranked by their reward $R(x)$. As metrics, we use the mean reward for these 100 sequences and the average pairwise edit distance among them as a measure of \emph{diversity}.

\paragraph{Results.}
We present the results 
in Table~\ref{tab:amp_results}, where we observe that GFlowNets trained with TB outperform all baselines on both performance and diversity metrics.

\section{Discussion and conclusion}

We introduced a novel training loss for GFlowNets, trajectory balance (TB), which yields faster and better training than the previously proposed flow matching (FM) and detailed balance (DB) losses. We proved that this objective, when minimized, yields the desired GFlowNet property of sampling from the target distribution specified by an unnormalized reward function. This new loss was motivated by the observation that the FM and DB losses are local in the action sequence and may require many iterations for credit assignment to propagate to early actions: if a gradient update introduces a flow inconsistency at some state far from the initial state (such as when a novel high-reward state is sampled), the parent of this state must be visited before the update is propagated closer to the root, akin to the slow propagation of reward signals in temporal difference learning.

We empirically found that TB discovered more modes of the energy function faster and was more robust than FM and DB to the exponential growth of the state space, due in part to the lengths of sequences and in part to the size of the action space. 
A factor to consider when interpreting our experimental results is that because we use a neural net rather than a tabular representation of policies, the early states' transitions are informed by downstream credit assignment via parameter sharing. Early states also get many more visits because there are more possible states near the ends of sequences than near the initial state. Finally, TB trades off the advantage of immediately providing credit to early states with the disadvantage of relying on sampling of long trajectories and thus a potentially higher variance of the stochastic gradient. The high gradient variance is a possible limitation of TB in difficult environments, and ways to overcome it by interpolating between local and trajectory-level objectives have been studied in subsequent work \cite{subtb}.

All in all, we found that trajectory balance is a superior training objective in a broad set of experiments, making it the default choice for future work on GFlowNets.

\newpage 

\section*{Acknowledgments}

This research was enabled in part by computational resources provided by Compute Canada. All authors are funded by
their primary academic institution. We also acknowledge funding from CIFAR, Samsung, IBM, Microsoft, and the Banting Postdoctoral Fellowship.

The authors are grateful to all the members of the Mila GFlowNet group, in particular to Dinghuai Zhang, for many fruitful research discussions, as well as to Yiheng Zhu for feedback on the published code. We also thank the anonymous reviewers for their comments.

\bibliography{ref}
\bibliographystyle{style/icml2022}

\section*{Checklist}

\begin{enumerate}

\item For all authors...
\begin{enumerate}
  \item Do the main claims made in the abstract and introduction accurately reflect the paper's contributions and scope?
    \answerYes{}
  \item Did you describe the limitations of your work?
    \answerYes{This is mainly a paper about theory and algorithms. Future work that applies these algorithms, in particular for domains where they can most immediately have an impact (e.g., molecule design for drug discovery), should consider the limitations and negative societal impacts of these applications.}
  \item Did you discuss any potential negative societal impacts of your work?
    \answerNA{See previous.}
  \item Have you read the ethics review guidelines and ensured that your paper conforms to them?
    \answerYes{}
\end{enumerate}

\item If you are including theoretical results...
\begin{enumerate}
  \item Did you state the full set of assumptions of all theoretical results?
    \answerYes{}
        \item Did you include complete proofs of all theoretical results?
    \answerYes{See \S\ref{app:proofs}.}
\end{enumerate}

\item If you ran experiments...
\begin{enumerate}
  \item Did you include the code, data, and instructions needed to reproduce the main experimental results (either in the supplemental material or as a URL)?
    \answerYes{For the grid and molecule environments.} \answerNo{For the other environments.}
  \item Did you specify all the training details (e.g., data splits, hyperparameters, how they were chosen)?
    \answerYes{See the Appendix.}
        \item Did you report error bars (e.g., with respect to the random seed after running experiments multiple times)?
    \answerYes{See the Appendix.}
        \item Did you include the total amount of compute and the type of resources used (e.g., type of GPUs, internal cluster, or cloud provider)?
    \answerYes{See the Appendix.}
\end{enumerate}

\item If you are using existing assets (e.g., code, data, models) or curating/releasing new assets...
\begin{enumerate}
  \item If your work uses existing assets, did you cite the creators?
    \answerYes{See the relevant experiment sections.}
  \item Did you mention the license of the assets?
    \answerNA{}
  \item Did you include any new assets either in the supplemental material or as a URL?
    \answerNA{References to the molecule and AMP sequence data are provided.}
  \item Did you discuss whether and how consent was obtained from people whose data you're using/curating?
    \answerNA{No new data collection was done.}
  \item Did you discuss whether the data you are using/curating contains personally identifiable information or offensive content?
    \answerNA{Not relevant for the domains studied.}
\end{enumerate}

\item If you used crowdsourcing or conducted research with human subjects...
\begin{enumerate}
  \item Did you include the full text of instructions given to participants and screenshots, if applicable?
    \answerNA{}
  \item Did you describe any potential participant risks, with links to Institutional Review Board (IRB) approvals, if applicable?
    \answerNA{}
  \item Did you include the estimated hourly wage paid to participants and the total amount spent on participant compensation?
    \answerNA{}
\end{enumerate}

\end{enumerate}

\newpage
\appendix

\section{Theoretical appendix}
\subsection{Proof of Proposition 1}
\label{app:proofs}

Recall Proposition~\ref{prop:main}:
\begin{proposition*}
\mainproposition
\end{proposition*}
\begin{proof}
Part (a) is an elementary manipulation of the trajectory balance constraint (\ref{eqn:tb_constraint}), with $R(x)$ substituted for $F(x)$ by the reward matching assumption (\ref{eqn:reward_matching}).

Conversely, if $\LTB(\tau)=0$ for all complete trajectories $\tau=(s_0\ra\to\dots\to s_n=x)$, then the policies $P_F(-|-;\theta)$ and $P_B(-|-;\theta)$ satisfy the constraint
\begin{equation}
    Z\prod_{t=1}^nP_F(s_t|s_{t-1};\theta)=R(x)\prod_{t=1}^nP_B(s_{t-1}|s_t;\theta).
    \label{eqn:tb_constraint_reward}
\end{equation}
Let $x$ be a terminal state. By iterating the law of total probability, we have
\begin{equation}
    \sum_{\tau=(s_0\ra s_1\ra\dots\ra s_n=x)}\prod_{t=1}^nP_B(s_{t-1}|s_t;\theta)=1.
    \label{eqn:total_probability}
\end{equation}
(Each term in this sum is the conditional likelihood of $\tau$ conditioned on terminating at $s_n=x$ under the the Markovian flow $F_\theta'$ uniquely determined by setting terminal state flows $F_\theta'(x)=R(x)$ and backward policy $P_B(-|-;\theta)$, cf.\ the uniqueness properties.)

On the other hand, we have
\begin{align*}
    F_\theta(x)
    &= \sum_{\tau=(s_0\ra\dots\ra s_n)=x}F_\theta(\tau) & \text{(by definition of state flows)}\\
    &= \sum_{\tau=(x_0\ra\dots\ra s_n)=x}Z\prod_{t=1}^{n}P_F(s_t|s_{t-1};\theta) & \text{(by (\ref{eqn:flow_pf}))}\\
    &= \sum_{\tau=(x_0\ra\dots\ra s_n)=x}R(x)\prod_{t=1}^nP_B(s_{t-1}|s_t;\theta) & \text{(by (\ref{eqn:tb_constraint_reward}))} \\
    &=R(x) & \text{(by (\ref{eqn:total_probability}))}.
\end{align*}
We conclude that $F_\theta$ satisfies (\ref{eqn:reward_matching}), as desired. 

(We remark that one can show in a similar way that $F_\theta(s\ra s')=F_\theta'(s\ra s')$ for all actions $(s\ra s')\in\AAA$, and thus, by the uniqueness properties, $F_\theta=F_\theta'$, i.e., the forward and backward policies determine the same Markovian flow.)
\end{proof}

\subsection{Generalizations}
\label{app:extensions}

The trajectory balance constraint (\ref{eqn:tb_constraint}) can be generalized to partial (not complete) trajectories, i.e., those that do not start at $s_0$ and end in a terminal state. Generalizations such as those we present here could be useful for a future goal of modularized or hierarchical GFlowNets, where each module (or low-level GFlowNet) can apply them to just the subsequences they have access to (cf.\ \S9.4 and \S10.2 in \cite{bengio2021foundations}).

\paragraph{Subtrajectory balance.} If $\tau=(s_m\ra s_{m+1}\ra\dots\ra s_n$) is a partial trajectory (i.e., $(s_t\ra s_{t+1})\in\AAA$ for all $t$), then, for any Markovian flow $F$ with forward and backward policies $P_F$ and $P_B$,
\begin{equation}
F(s_m)\prod_{t=m}^{n-1}P_F(s_{t+1}\mid s_t)=F(s_n)\prod_{t=m}^{n-1}P_B(s_t\mid s_{t+1}).
\label{eqn:subtb_constraint}
\end{equation}
This can be derived by showing that both sides are equal to
\begin{equation}\sum_{\tau=(\dots\ra s_m\ra s_{m+1}\ra\dots\ra s_n\ra\dots)\in\TTT}F(\tau).\end{equation}

The trajectory balance constraint (\ref{eqn:tb_constraint}) is the special case of this for full trajectories, while the detailed balance constraint (\ref{eqn:db_constraint}) is the special case of trajectories wth only one edge. This subtrajectory balance constraint can be converted into a learning objective: a model can output estimated state flows $F_\theta(s)$ only for certain nonterminal states $s$ (``hubs"), and the error in (\ref{eqn:subtb_constraint}) optimized along segments of trajectories between these hubs. Thus the detailed balance loss corresponds to all nodes being hubs, and the trajectory balance loss corresponds to only the initial state $s_0$ being a hub.

Subtrajectory balance has been explored and was shown to have convergence benefits in a work that appeared while this paper was under review \cite{subtb}.

\paragraph{Non-forward trajectories.}
Trajectory balance has a more general form for trajectories that have a mix of forward and backward steps. Here we describe just one example: terminal-terminal paths that take several backward steps, then take several forward steps.

Let $s_1=s'_1$ be any state (not necessarily a child of the GFlowNet's initial state $s_0$) and $(s_1\ra s_2\ra\dots\ra s_n)$ and $(s'_1\ra s'_2\ra\dots\ra s'_{n'})$ two trajectories from $s_0$ to terminal states. Then the following must hold for any Markovian flow $F$:
\begin{equation}
R(s'_{n'})\prod_{t=1}^{n'-1} P_B(s'_t\mid s'_{t+1})\prod_{t=1}^{n-1} P_F(s_{t+1}\mid s_t)=
R(s_n)\prod_{t=1}^{n-1} P_B(s_t \mid s_{t+1}) \prod_{t=1}^{n'-1} P_F(s'_{t+1} \mid s'_t).
\label{eqn:back_forth_tb_constraint}
\end{equation}
That is, the path that goes ``backward, then forward” from $s_n$ to $s'_{n'}$ must have the same likelihood no matter in which direction it is traversed, up to the ratio of rewards at the endpoints. A simple way to derive (\ref{eqn:back_forth_tb_constraint}) is by writing the trajectory balance constraint for two paths from the GFlowNet's initial state to $s_n$ and $s'_{n'}$ that are identical until $s_1$ and then diverge, then dividing one constraint by the other. Notice that the flow $F(s_1)$ is not present here. Thus, (\ref{eqn:back_forth_tb_constraint}) can be converted into an learning objective does not require a model to output any state flows (even the initial state flow $Z$). 

Such terminal-terminal paths could also be used for exploration of $\XXX$ with MCMC-like local search algorithms \cite{zhang2022generative}. The special case of `one step back, two steps forward' paths was used for a graph generation problem in Bayesian structure learning \cite{deleu2022bayesian}.

\subsection{GFlowNets and variational methods}
\label{app:vi}

\newcommand{\bq}[1]{\left[#1\right]}
\newcommand{\pq}[1]{\left(#1\right)}
\let\f\frac
\newcommand{\Var}{{\rm Var}}

We build a connection between the TB loss for GFlowNets and a na\"ive variational approach to fitting sequential samplers.

Suppose that the backward policy $P_B$ is fixed, and suppose for ease of the derivation that it is known that $\sum_{x\in\XXX}R(x)=1$. As in the main text, we write $P_B(\tau|x)$ for the likelihood of obtaining the reverse of the trajectory tau when sampling from the backward policy starting from $x$. 

The on-policy trajectory balance objective has gradient with respect to the parameters of $P_F$:
\begin{equation}
    \EE_{(\tau,x) \sim P_F}\bq{\nabla_\theta\pq{\log\f{R(x)P_B(\tau|x)}{P_F(\tau;\theta)}}^2}
    =
    \EE_{(\tau,x) \sim P_F}\bq{-2\log\f{R(x)P_B(\tau|x)}{P_F(\tau;\theta)}\nabla_\theta\log P_F(\tau;\theta)},
    \label{eqn:tb_gradient}
\end{equation}
which is estimated by sampling $\tau\sim P_F$ (terminating in $x\in\XXX$) and computing the term inside the expectation.

The model $P_F$ could also be optimized with respect to an evidence lower bound (ELBO) objective, i.e., minimizing $\KL(P_F(\tau) \| R(x)P_B(\tau|x))$. We derive the gradient of this KL:
\begin{align}
    &\nabla_\theta\KL(P_F(\tau;\theta)\|R(x)P_B(\tau|x))\\
    &=\nabla_\theta\EE_{(\tau,x)\sim P_F}\bq{\log\f{P_F(\tau)}{R(x)P_B(\tau|x)}}\nonumber\\
    &=\EE_{(\tau,x)\sim P_F}\bq{\nabla_\theta\log\f{P_F(\tau;\theta)}{R(x)P_B(\tau|x)}+\log\f{P_F(\tau;\theta)}{R(x)P_B(\tau|x)}\nabla_\theta\log P_F(\tau;\theta)}.
    \label{eqn:kl_gradient}
\end{align}
The last step is the standard score function trick, and the Reinforce estimator optimizes the KL by sampling $\tau \sim P_F(\tau;\theta)$ and using the term inside the expectation as the direction of the gradient step. But now notice that
\begin{equation}
    \EE_{(\tau,x)\sim P_F}\bq{\nabla_\theta\log\f{P_F(\tau;\theta)}{R(x)P_B(\tau|x)}}
    = 
    \EE_{(\tau,x)\sim P_F}\bq{\nabla_\theta\log P_F(\tau;\theta)}=0,
\end{equation}
because of the constraint $\sum_\tau P_F(\tau;\theta)=1$. We conclude that  the expected trajectory balance gradient (\ref{eqn:tb_gradient}) is equal to the expected Reinforce gradient (\ref{eqn:kl_gradient}) up to a constant.

However in the vicinity of the optimum (when TB and KL equal 0), the TB graident estimator has lower variance, as the following computation shows:
\begin{align*}
    &\Var_{(\tau,x)\sim P_F}\bq{\log\f{p(x)}{P_F(\tau;\theta)}\nabla_\theta\log P_F(\tau;\theta)}
    \\&-\Var_{(\tau,x)\sim P_F}\bq{\nabla_\theta\log\f{P_F(\tau;\theta)}{R(x)P_B(\tau|x)}+\log\f{P_F(\tau;\theta)}{R(x)P_B(\tau|x)}\nabla_\theta\log P_F(\tau;\theta)}\\
    &=-\EE_{(\tau,x)\sim P_F}\bq{(\nabla_\theta\log P_F(\tau;\theta) \nabla_\theta\log P_F(\tau;\theta)^\top)\pq{1+2\log\f{R(x)P_B(\tau|x)}{P_F(\tau;\theta)}}}.
\end{align*}
If the term in parentheses is always positive (in particular, in the neighbourhood of the solution where $P_F(\tau)=R(x)P_B(\tau|x)$ for all $\tau$), then the difference of variances for all directional derivatives is negative for all $\theta$.

The connection between GFlowNets and variational methods was more thoroughly explored in two papers that appeared while this work was under review \cite{unifying,gfn-hvi}.

\section{Experimental appendix}

\subsection{Hypergrid}
\label{app:grid}

For the GFlowNet policy model, we use an MLP of the same architecture as \cite{bengio2021flow}, with 2 hidden layers of 256 hidden units each. We train all models with a learning rate of 0.001 ($P_F$ and $P_B$ policy model) and 0.1 ($Z_\theta$)  with up to $10^6$ sampled trajectories with a batch size of 16, using the Adam optimizer with all other parameters at their default values.

To reproduce the flow matching and non-GFlowNet baselines, we used the code published by \cite{bengio2021flow} out of the box. For TB and DB, we used a learning rate of $10^{-3}$ for the flow and policy models and a $10^{-1}$ for the initial state flow $\log Z=\log F(s_0)$. (In a search of learning rates over powers of 10, $10^{-3}$ was found to be the largest that does not lead to instability in the form of rapid mode collapse.) All experiments with TB and DB were performed on CPU and take about 2 hours for $10^6$ episodes on a single core, totaling $\sim$10 CPU days for all 24 DB and TB experiment settings with 5 seeds each:
\[
\{{\rm TB},{\rm DB}\} \times \{\text{uniform $P_B$},\text{learned $P_B$}\}\times \{R_0=10^{-1},10^{-2},10^{-3}\}\times\{(H,d)=(8,4),(64,2)\}.
\]

\subsection{Molecule synthesis}
\label{app:mol_synth}
We use the dataset and proxy model provided by \cite{bengio2021flow}. We also train GFlowNet using the same architecture and hyperparameters (except $\beta$ and learning rate) but using the trajectory balance loss presented in this paper, using fixed uniform backward policy $P_B$. The binding scores in the provided dataset were computed with AutoDock~\citep{trott2010autodock}.

To test hyperparameter robustness we trained models using reward exponents $\beta = \{4, 8, 10, 16\}$, and learning rates $\{5\times 10^{-5}, 10^{-4}, 5\times 10^{-4}, 10^{-3}\}$. In contrast to \cite{bengio2021flow}, we used a more exploratory training policy: with probability $0.1$ (instead of the original $0.05$) trajectories are set to stop at some length $k$, which is chosen uniformly between 3 and 8, the minimum and maximum allowed trajectory length respectively.

We observed a runtime improvement of up to $5\times$ for TB relative to FM. There are three factors responsible for this:
\begin{enumerate}[nosep,label=(\arabic*),left=0pt]
\item Most importantly, FM requires as many model evaluations as there are parents of all states in a sampled trajectory, since the model gives the out-flows $F(s\ra s')$ for an input state $s$, while the objective involves a sum over in-flows. On the other hand, TB and DB require just one evaluation of the forward and backward policy models per state.
\item The average trajectory length. If a model learns to terminate early with higher frequency, trajectories are shorter and fewer model evaluations are required.
\item  Hardware and the ratio of CPU and GPU load. Experiments on the molecule domain were run on a Tesla K80 GPU; the computation time benefit of TB appears to be smaller but still present on newer hardware with identical batch size settings. (Meanwhile, experiments on the lightweight grid domain were run on CPU, and the trajectory length was the main factor controlling runtime.)
\end{enumerate}

\subsection{Bit sequence generation}
\label{app:bit_sequence}
\textbf{Generating reference sequences}.
Let $H$ be a set of symbols (short bit sequences of length $b$), for instance $H = \{0110, 1100, 1111, 0000, 0011\}$. Sequences in $S$ are then constructed by randomly combining $m$ symbols from $H$, for instance, $0011110000000011$ where $m=4$. This construction imposes a structure on $R(x)$. In our experiments we set $m=\frac{n}{b}$, $b=8$, $H=\{'00000000', '11111111', '11110000', '00001111', '00111100'\}$. 

\textbf{Generating the test set}.
Since the reward is defined based on the edit distance from the sequences in set $M$, we generate a test set sampled approximately uniformly over the possible values of $R(x)$ as follows: (1) pick a mode $s\in M$, (2) modify $i$ bits randomly $\forall i<n$ and we repeat this for all the modes.

\textbf{Implementation}.
We implement GFlowNets with TB and FM in PyTorch\footnote{Code: \url{https://gist.github.com/MJ10/59bfcc8bce4b5fce9c1c38a81b1105ae}.} for autoregressive generation tasks, along with the A2C baseline. For the MARS (MCMC) baseline we modify the implementation released by~\cite{bengio2021flow}.

\textbf{Hyperparameters}.
We use a Transformer \citep{vaswani2017attention} as the neural network architecture for all the methods. We use 3 hidden layers with hidden dimension $64$ with $8$ attention heads. All methods were trained for $50,000$ iterations, with a minibatch size of $16$. We set the the random action probability to $0.0005$ selected from $\{0.0001, 0.0005, 0.001, 0.01\}$, the reward exponent $\beta$ to $3$ selected from $\{2, 3, 4\}$, and the sampling temperature for $P_F$ to $1$ for the GFlowNets. For trajectory balance we use a learning rate of  $1\times 10^{-4}$ selected from $\{10^{-5}, 10^{-4}, 5\times 10^{-4}, 10^{-3}, 5\times 10 ^ {-3}\}$ for the policy parameters and $1\times 10^{-3}$ for $\log Z$. For flow matching we use a learning rate of $5\times 10^{-4}$ selected from $\{10^{-5}, 10^{-4}, 5\times 10^{-4}, 10^{-3}, 5\times 10 ^ {-3}\}$ with leaf loss coefficient $\lambda_T = 10$. For A2C with entropy regularization we share parameters between the actor and critic networks, and use learning rate of $10^{-4}$ selected from $\{10^{-5}, 10^{-4}, 5\times 10^{-4}, 10^{-3}, 5\times 10 ^ {-3}\}$ with entropy regularization coefficient $10^{-3}$ selected from $\{10^{-4}, 10^{-3}, 10 ^ {-2}\}$. For SAC we use the formulation in~\cite{christodoulou2019soft} with a learning rate of $5\times 10^{-4}$ selected from $\{10^{-5}, 10^{-4}, 5\times 10^{-4}, 10^{-3}, 5\times 10 ^ {-3}\}$ target network update frequency $500$ and $200$ initial random steps. For the MARS baseline we set the learning rate to $5\times 10^{-4}$ selected from $\{10^{-5}, 10^{-4}, 5\times 10^{-4}, 10^{-3}, 5\times 10 ^ {-3}\}$. For all the methods we use the Adam optimizer. Overall, for the Bit sequence generation experiments we used $25$ GPU days.

\subsection{AMP generation}
\label{app:amp_gen}

\textbf{Vocabulary}.
The vocabulary of the 20 amino acids is defined as: \texttt{[`A', `C', `D', `E', `F', `G', `H', `I', `K', `L', `M', `N', `P', `Q', `R', `S', `T', `V', `W', `Y']}

\textbf{Reward Model}.
We use a Transformer-based classifier, with $4$ hidden layers, hidden dimension $64$, and $8$ attention heads. We train it with a minibatch of size $256$, with learning rate $10^{-4}$, with early stopping on the validation set. 

\textbf{Hyperparameters}.
As with the bit sequences, we use a Transformer \cite{vaswani2017attention} as the neural network architecture for all the methods. We use 3 hidden layers with hidden dimension $64$ with $8$ attention heads. All method were trained for $20,000$ iterations, with a mini batch size of $16$. We set the the random action probability to $0.01$ selected from $\{0.0001, 0.0005, 0.001, 0.01\}$, the reward exponent $\beta: R(x)^\beta$ to $3$ selected from $\{2, 3, 4\}$, and the sampling temperature for $P_F$ to $1$ for the GFlowNets. For trajectory balance we use a learning rate of  $5\times 10^{-3}$ selected from $\{10^{-5}, 10^{-4}, 5\times 10^{-4}, 10^{-3}, 5\times 10 ^ {-3}\}$ for the flow parameters and $1\times 10^{-2}$ for $\log Z$. For flow matching we use a learning rate of $5\times 10^{-4}$ selected from $\{10^{-5}, 10^{-4}, 5\times 10^{-4}, 10^{-3}, 5\times 10 ^ {-3}\}$ with leaf loss coefficient $\lambda_T = 25$. For A2C with entropy regularization we share parameters between the actor and critic networks, and use learning rate of $5\times 10^{-4}$ selected from $\{10^{-5}, 10^{-4}, 5\times 10^{-4}, 10^{-3}, 5\times 10 ^ {-3}\}$ with entropy regularization coefficient $10^{-2}$ selected from $\{10^{-4}, 10^{-3}, 10 ^ {-2}\}$. For SAC we use the formulation in~\cite{christodoulou2019soft} with a learning rate of $5\times 10^{-4}$ selected from $\{10^{-5}, 10^{-4}, 5\times 10^{-4}, 10^{-3}, 5\times 10 ^ {-3}\}$ target network update frequency $400$ and $200$ initial random steps. For the MARS baseline we set the learning rate to $5\times 10^{-4}$ selected from $\{10^{-5}, 10^{-4}, 5\times 10^{-4}, 10^{-3}, 5\times 10 ^ {-3}\}$. We run the experiments on 3 seeds and report the mean and standard error over the three runs in Table~\ref{tab:amp_results}. Overall, for the AMP Generation experiments we used $14$ GPU days. 

\end{document}